\documentclass{article}

 \usepackage[preprint, nonatbib]{neurips_2026}

\usepackage[utf8]{inputenc} 
\usepackage[T1]{fontenc}    
\usepackage{hyperref}       
\usepackage{url}            
\usepackage{booktabs}       
\usepackage{amsfonts}       
\usepackage{nicefrac}       
\usepackage{microtype}      
\usepackage{xcolor}         

\usepackage{todonotes}
\usepackage{amsmath, amsthm, amsxtra, amssymb, amstext, mathtools}
\usepackage{thmtools} 
\usepackage{xspace}
\usepackage{enumitem}
\usepackage[algo2e,ruled,vlined,linesnumbered]{algorithm2e}\SetArgSty{upshape}
\usepackage{siunitx}

\usepackage{subcaption}
\usepackage{thm-restate}
\usepackage[nameinlink]{cleveref}

\usepackage[shortcuts]{extdash} 

\newcommand{\mylineref}[1]{\hyperref[#1]{line~\ref*{#1}}}

\newcommand{\NSGA}{\mbox{NSGA}\nobreakdash-II\xspace}
\newcommand{\moead}{MOEA/D\xspace}
\newcommand{\selfAdjustingMoead}{sMOEA/D\xspace}

\newcommand{\oneminmax}{\textsc{OneMinMax}\xspace}

\newcommand{\lotz}{\textsc{LOTZ}\xspace}
\newcommand{\monotoneLotz}{\textsc{LOTZ}\textrm{$_{u, v}$}\xspace}

\newcommand{\zeroOneKnapsack}{$0$/$1$-\textsc{Knapsack}\xspace}

\newcommand{\R}{\ensuremath{\mathbb{R}}}

\newcommand{\N}{\ensuremath{\mathbb{N}}} 

\let\originalleft\left
\let\originalright\right
\renewcommand{\left}{\mathopen{}\mathclose\bgroup\originalleft}
\renewcommand{\right}{\aftergroup\egroup\originalright}

\makeatletter
\@ifundefined{newcounteralias}{}{%
  \renewcommand\thmt@autorefsetup{%
    \@xa\def\csname\thmt@envname autorefname\@xa\endcsname
      \@xa{\thmt@thmname}%
  }%
}
\makeatother

\newtheorem{theorem}{Theorem}
\newtheorem{lemma}[theorem]{Lemma}

\title{A Dynamic Decomposition Strategy for the MOEA/D With Proven Performance Guarantees}
\author{
  Benjamin Doerr$^1$, Martin~S. Krejca$^1$, Noé Weeks$^2$\\
  $^1$ Laboratoire d'Informatique (LIX), CNRS, École Polytechnique,\\ Institut Polytechnique de Paris, Palaiseau, France\\
  $^2$ École Normale Supérieure, Paris, France\\
  \texttt{firstname.lastname@polytechnique.edu}, \texttt{noe.weeks@ens.psl.eu}
}

\begin{document}

\maketitle

\begin{abstract}
  The MOEA/D is one of the most successful multi-objective evolutionary algorithms. It operates by decomposing the optimization problem into single-objective subproblems, which are then solved in a co-evolutionary manner. The performance of the MOEA/D depends critically on this decomposition.
  In this work, we propose a novel way to adjust this decomposition on the fly, for the case of discrete bi-objective optimization. Different from all previous works in this direction, we support our new algorithm with a proven performance guarantee.
  This is the first work that rigorously proves the runtime of a dynamic multi-objective evolutionary algorithm.
  For monotonic distortions of the simple LOTZ benchmark, it is known that classic decompositions leave large holes in the Pareto front, making it very hard for the MOEA/D to cover the entire Pareto front. In contrast, we prove that our self-adjusting MOEA/D solves all these LOTZ problems in an expected number of $O(n^4)$ function evaluations. This relatively low runtime as well as the details of our mathematical proof show that the dynamic decomposition strategy indeed evolves a decomposition adjusted to each of these problem instances.
\end{abstract}

\section{Introduction}
\label{sec:introduction}

Many real-world optimization problems feature several conflicting objectives. Such \emph{multi-objective optimization problems} (MOOPs) usually do not possess a single optimal solution. Thus, a common approach is to compute a diverse set of \emph{Pareto optima} (incomparable optimal trade-offs) and let a human decision maker pick one of these, based on their broader understanding of the problem.

Due to their population-based nature, genetic algorithms are one of the most prominent approaches for computing the Pareto front of a MOOP. The famous \NSGA algorithm~\cite{DebPAM02}, for example, is massively used in practice and has reached over $50\,000$ citations on Google Scholar. We refer to~\cite{ZhouQLZSZ11} for a broader introduction to the use of genetic algorithms in MOO.

There are two general approaches to the design of multi-objective evolutionary algorithms (MOEAs). The classic one works with a population of individuals and employs diversity mechanisms to ensure that this population approaches the full Pareto front. This concept is found in the algorithms SPEA2~\cite{ZitzlerLT01}, \NSGA~\cite{DebPAM02}, SMS-EMOA~\cite{BeumeNE07}, and NSGA-III~\cite{DebJ14}.

More recently, decomposition-based approaches were proposed. Here the idea is to decompose the MOOP into scalar subproblems and solve these in a co-evolutionary manner, that is, they are optimized in parallel and exchange solutions. The hope is that the co-evolutionary approach helps solve similar subproblems efficiently, but also that the evolutionary process finds Pareto optima that are not the optimum of one of the designated subproblems. We refer to a recent survey~\cite{XuXM20} on decomposition-based MOEAs and note that the by far most prominent such algorithm is the MOEA/D developed by Zhang and Li~\cite{ZhangL07}.
The survey reports that the MOEA/D with a static decomposition oftentimes leads to a worse spread of the final solutions in terms of objective values when compared to a dynamic choice.
This behavior shows repeatedly for continuous and for discrete combinatorial problems, and in the presence of crossover, oftentimes reported for the common Chebyshev decomposition~\cite{HuaJH19,IshibuchiSMN17,LiL11,LiY20,QiMLJSW14}.

Algorithms using one joint population are relatively well understood. In particular, many mathematical analyses were conducted, first for the simplistic (G)SEMO family of algorithms~(e.g.,~\cite{DoN21,Giel03,LaumannsTZ04,QianYZ15ijcai,Thierens03}), then for practically more relevant \NSGA~\cite{CerfDHKW23,ZhengLD22}, {NSGA-III}~\cite{OprisDNS24,WiethegerD23}, SMS-EMOA~\cite{BianZLQ23,ZhengD24}, and SPEA2~\cite{RenBLQ24}.

In contrast, the mathematical analysis of decomposition-based algorithms such as the MOEA/D is much less developed. Refraining from details, this area started only in 2016 with the seminal paper~\cite{LiZZZ16}, then seeing six follow-ups~\cite{DoerrKW24,DoNNS23,HuangZ20,HuangZCHLX21,HuangZCH19,HuangZLL21}.

With relatively few \moead theory works, it comes as no surprise that these works often take simplistic assumptions.
Most results assume that all Pareto optima appear also as optimum of one of the subproblems of the decomposition~\cite{HuangZCHLX21,HuangZLL21,LiZZZ16}; others employ non-standard variation operators~\cite{HuangZ20} or operators that are particularly suited for the problem at hand~\cite{HuangZCH19}.

Without such optimistic assumptions, the performance of the \moead is much less understood. Prominently, Li et~al.~\cite{LiZZZ16} showed that for certain monotonic rescalings of the common LOTZ benchmark, the common uniform decomposition results in subproblems whose optima cover very badly the Pareto front.
The authors argue that the missing Pareto optima are hard to find.
In addition, more recently, Doerr, Krejca, and Week~\cite{DoerrKW24} showed rigorously that the \moead can perform poorly on the simple \oneminmax benchmark if the number of subproblems is ill-suited to the Pareto front size and if a common mutation operator is used.

That a bad fit of the decomposition to the shape of the Pareto front can lead to drastic performance problems has been observed in several empirical works before, and several dynamic ways to adjust the decomposition exist. Due to their large number, we cannot detail them here and have to refer the reader to the survey~\cite{MaYLQZ20}. For none of them, proven performance guarantees exist.

\textbf{Our contribution:}
We propose a simple dynamic decomposition for the bi-objective MOEA/D on discrete search spaces and analyze the performance of the resulting \emph{self-adjusting \moead} (\selfAdjustingMoead; \Cref{alg:self-adjusting-moead}). In our dynamic decomposition scheme, some ideas present in other works can be found again, but overall we view it as a new development. In particular, since we aim at computing the full Pareto front of a discrete problem, we have to adjust the number of subproblems, whereas most previous schemes work with a static number of subproblems.
Furthermore, our \selfAdjustingMoead makes use of a particularly good choice for new subproblems (see \mylineref{line:update} in \Cref{alg:self-adjusting-moead} and \cref{eq:update}) that allows to find new useful subproblems quickly (see \Cref{lem:fit_incr}).

Most importantly, our \selfAdjustingMoead allows to be analyzed theoretically, resulting in the first mathematical runtime guarantees of a dynamic multi-objective evolutionary algorithm (MOEA); all previous runtime analyses for dynamic evolutionary algorithms only regarded single-objective algorithms~\cite{DoerrD20bookchapter}.
Due to the drastically increased complexity rigorously analyzing a dynamic algorithm, we focus on bi-objective problems and do not consider more advanced operators, such as crossover.
We note that such assumptions are standard practice for new research areas in this field and can be observed in the references we provided earlier and their references therein.

As our main result, we prove that the \selfAdjustingMoead solves \emph{all} monotonically distorted versions of the common LOTZ benchmark of size~$n$ in expected time $O(n^4)$ (\Cref{thm:lotz-run-time}).
This covers a huge range of different variants of the LOTZ benchmark and shows in a mathematically rigorous way that our \selfAdjustingMoead does not run into the problems detected by Li et~al.~cite{LiZZZ16} for the classic \moead.
Furthermore, this runtime guarantee is independent on the distortion strength, emphasizing that the exact shape of the objective function is not important for the \selfAdjustingMoead.

Overall, our work provides a promising solution for finding a suitable decomposition for the \moead on discrete problems, supported by the first mathematical runtime analysis of a dynamic MOEA.

\section{Setting and Terminology}
\label{sec:setting}

Let~$\N$ denote the set of all natural numbers (with~$0$), and let~$\R$ denote the set of all reals.
For all $m, n \in \R$, we define $[m .. n] \coloneqq [m, n] \cap \N$ as well as $[n] \coloneqq [1 .. n]$. For a vector $v \in \R^d$, we denote its components at index $i \in [d]$ by~$v_i$.

We consider bi-objective pseudo-Boolean maximization, that is, given an integer $n \in \N_{\geq 1}$, we consider the maximization of a (bi-objective) \emph{objective function} $f = (f_1,f_2) \colon \{0, 1\}^n \to \R^2$.
We call each $x \in \{0, 1\}^n$ an \emph{individual}, and~$f(x)$ the \emph{objective value of~$x$}.

We compare objective values via the \emph{weak domination} weak partial order~$\succeq$.
For all $u, v \in \R^2$, we say that \emph{$u$ weakly dominates~$v$} (written as $u \succeq v$) if and only if $u_1 \geq v_1$ and $u_2 \geq v_2$.
We extend this terminology to individuals of an objective function~$f$ and say, for all $x, y \in \{0, 1\}^n$, that \emph{$x$ weakly dominates~$y$} if and only if $f(x) \succeq f(y)$. We speak of \emph{domination}, denoted by $\succ$, if we have weak domination and at least one of the inequalities is strict.

We are interested in the set of all non-dominated objective values, known as the \emph{Pareto front (of~$f$)}, that is, the set $F(f) \coloneqq \{f(x) \mid x \in \{0, 1\}^n \land \nexists y \in \{0, 1\}^n \colon f(y) \succ f(x)\}$.
We call each value in~$F(f)$ a \emph{Pareto optimum}.

\begin{figure}[t]
  \begin{subfigure}[t]{.48 \textwidth}
    \includegraphics[width = \textwidth]{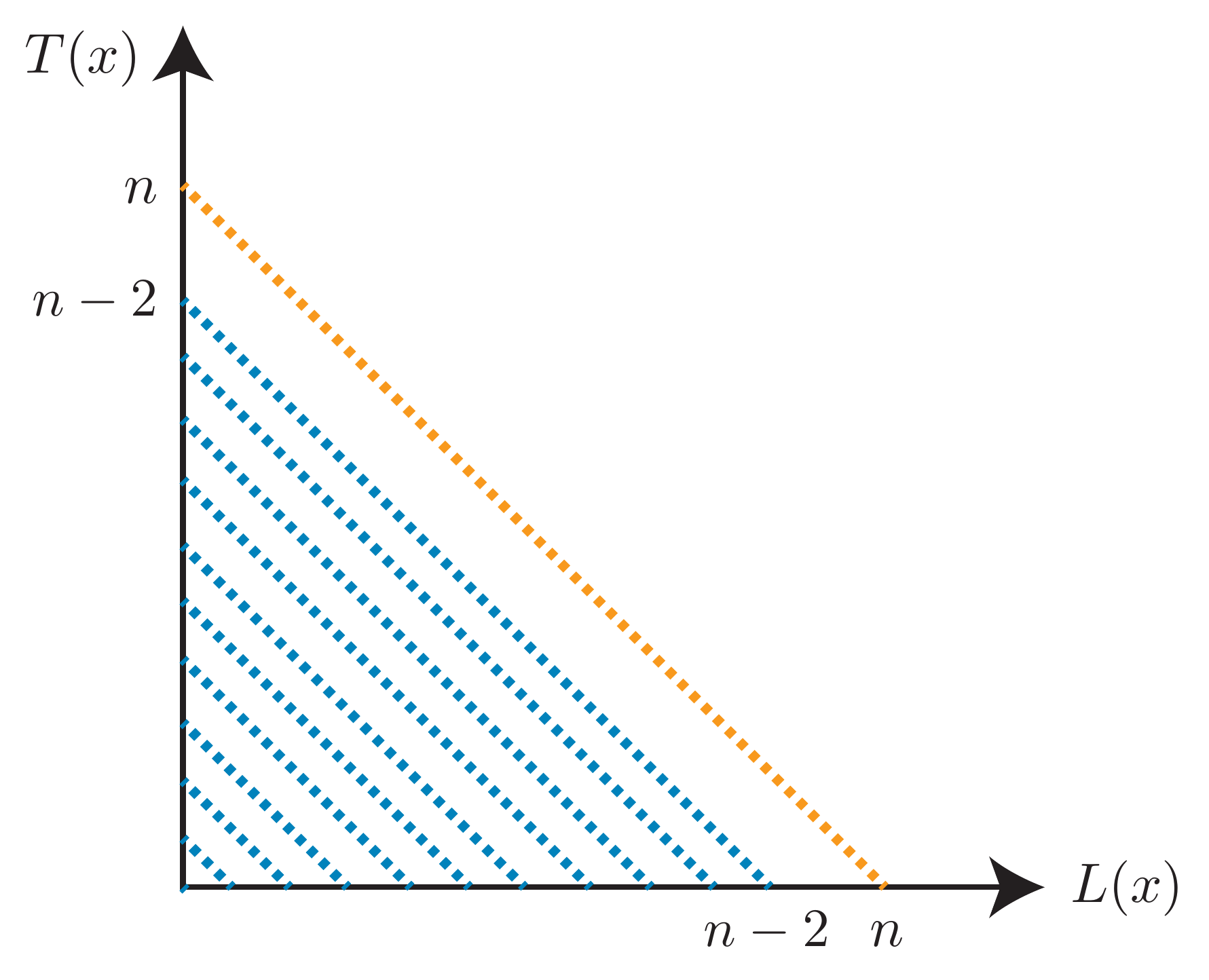}
    \caption{
      The original \lotz benchmark.
      The Pareto front consists of all bit strings that start with a consecutive number of~$1$s and then continues with a consecutive number of~$0$s.
      All worse solutions have one or more~$0$s followed by one or more~$1$s.
    }
    \label{fig:lotz-versions:lotz}
  \end{subfigure}
  \hfil
  \begin{subfigure}[t]{.48 \textwidth}
    \includegraphics[width = \textwidth]{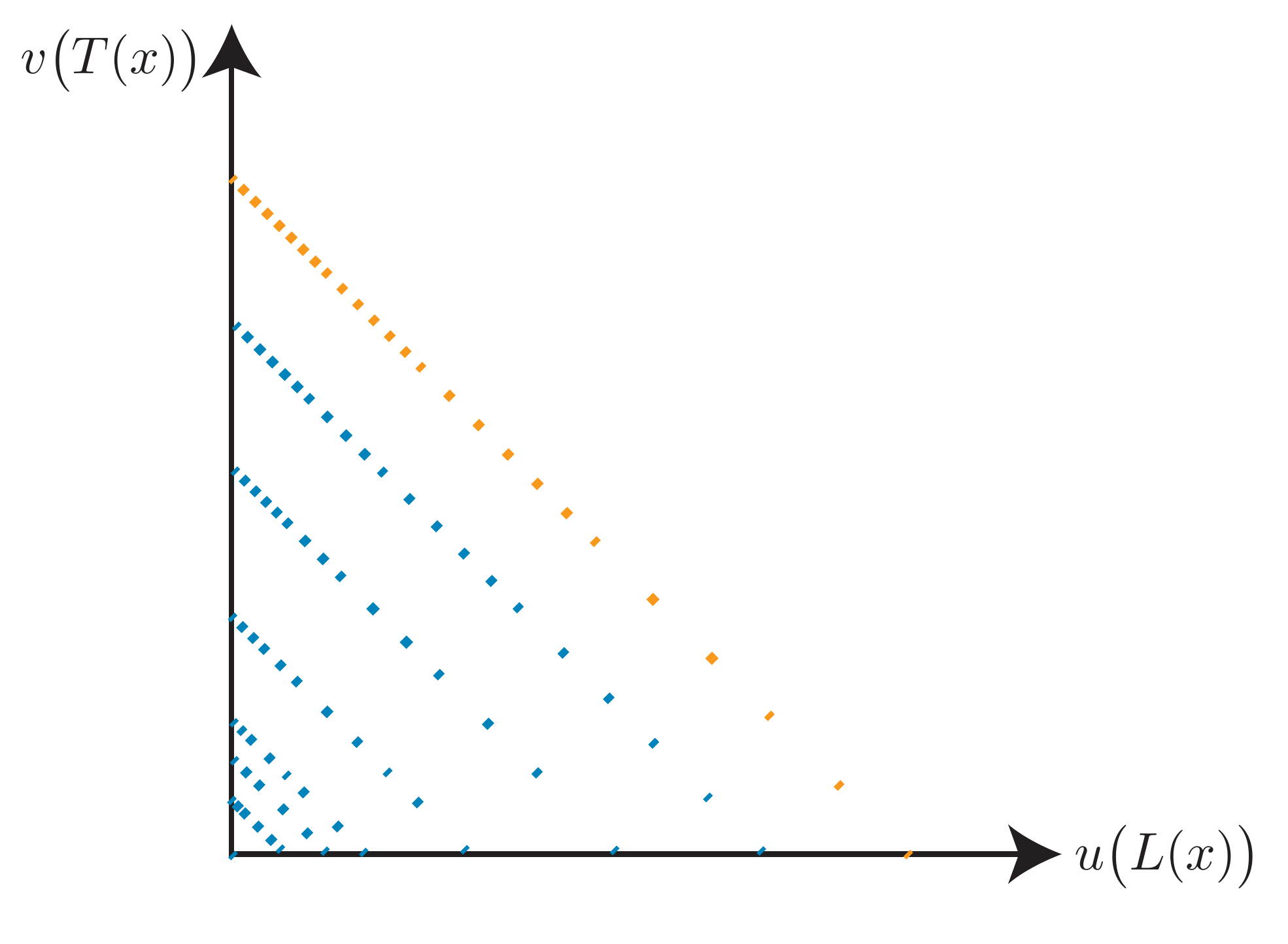}
    \caption{
      The monotonically distorted version of \lotz scales both objectives of \lotz, not necessarily in the same way, in a monotonic fashion.
      The Pareto front is not changed, but the values of the objective values and the gaps between them can differ drastically to the classic version.
    }
    \label{fig:lotz-versions:scaled-lotz}
  \end{subfigure}
  \caption{
    A graphical sketch of the LeadingOnesTrailingZeros (\lotz)~\cite{LaumannsTZ04} benchmark (\Cref{fig:lotz-versions:lotz}) and its monotonically distorted version~\cite{LiZZZ16} (\Cref{fig:lotz-versions:scaled-lotz}).
    The details are given in \Cref{sec:setting}.
    The orange points on the outmost line denote the Pareto front of the respective benchmark.
    The other blue points denote other objective values that are not Pareto optima.
  }
  \label{fig:lotz-versions}
\end{figure}

\paragraph{Objective function.}
We consider a generalization of the LeadingOnesTrailingZeros (\lotz)~\cite{LaumannsTZ04} benchmark, which is an established benchmark for the theoretical analysis of multi-objective heuristics.
It aims to optimize the consecutive number of leading~$1$s as well the consecutive number of trailing~$0$s in a bit string.
Formally, let $L, T\colon \{0, 1\}^n \to \R$ be such that $L\colon x \mapsto \sum_{i \in [n]} \prod_{j \in [i]} x_j$ and $T\colon x \mapsto \sum_{i \in [n]} \prod_{j \in [n - i + 1 .. n]} (1 - x_i)$, that is,~$L$ yields the number of leading~$1$s, and~$T$ the number of trailing~$0$s.
Then $\lotz$ is the bi-objective function $\lotz\colon \{0, 1\}^n \to \R^2, x \mapsto \bigl(L(x), T(x)\bigr)$ (\Cref{fig:lotz-versions:lotz}).
We consider monotonically distorted versions of \lotz, based on~\cite{LiZZZ16} (\Cref{fig:lotz-versions:scaled-lotz}).
Let $u, v\colon \R_{\geq 0} \to \R_{\geq 0}$ be monotonically increasing functions.
Then we define
\begin{align*}
  \monotoneLotz\colon \{0, 1\}^n \to \R^2, x \mapsto \bigl(u\bigl(L(x)\bigr), v\bigl(T(x)\bigr)\bigr).
\end{align*}

Independent of $u$ and $v$, we have $F(\monotoneLotz) = \{(1^k, 0^{n - k}) \mid k \in [0 .. n]\}$, which has size $n + 1$.
For all $x \in \{0, 1\}^n$ and $k \in [0 .. n]$ with $x = 1^k 0^{n - k}$, we call~$x$ a \emph{$(k, n - k)$-individual}.

Throughout this section, we let $n \in \N_{\geq 1}$ always denote the \emph{problem size} of \monotoneLotz, that is, the length of the bit strings, and we assume that~$n$ is given implicitly.
Similarly, we assume that~$u$ and~$v$ are arbitrary and implicitly given.
Moreover, big-O notation always refers to asymptotics in~$n$.

\paragraph{Fitness level method.}
In parts of our analysis, we make use of the fact that the \selfAdjustingMoead can only improve on certain values, such as the fitness of a best-so-far individual for a subproblem.
In order to estimate the expected waiting time until sufficiently many improvements occurred, we sum the expected times for each improvement.
As each trial for an improvement is independent of the other trials, the probability for an improvement follows a geometric law.
Bounding the improvement probability of such a law from below results in an upper bound of its expected waiting time, allowing us to bound the overall waiting time from above.

In the field of evolutionary computing, this argument is called \emph{fitness level method}, and it was introduced by \cite{Wegener01}.
The method works for any real-valued Markov chain~$(X_t)_{t \in \N}$ over a finite space~$S$ when given a potential function $h\colon S \to \R$ and a target value $v \in \R$.
We aim to bound the expected time $T = \inf \{t \in \N \mid h(X_t) = v\}$.
It is important that the potential cannot decrease, that is, for all $t \in \N$, we have $h(X_t) \leq h(X_{t + 1})$.
We call the range of~$h$ the \emph{fitness levels} (of~$h$).
Formally, we have the following theorem, which is a reformulation of a version by \cite{Sudholt13}.

\begin{theorem}[Rephrased from {\protect\cite[Theorem~$2$]{Sudholt13}}]
  \label{thm:fitnessLevelMethod}
  Let~$(X_t)_{t \in \N}$ be a real-valued Markov chain over a finite space~$S$, let $h\colon S \to \R$ be a potential function, let $v \in \R$ be a target value, and let $T = \inf \{t \in \N \mid X_t = v\}$.
  Furthermore, for all $i \in \{j \in \R \mid \exists s \in S\colon h(s) = j\} \smallsetminus \{v\} \eqqcolon I$, let~$q_i$ denote a lower bound on the probability of~$X$ transitioning from a state with potential~$i$ to a state with a strictly larger potential.
  Then $\mathrm{E}[T] \leq \sum\nolimits_{i \in I} \frac{1}{q_i}$.
\end{theorem}

\section{The \moead}
\label{sec:algorithms}

\begin{algorithm2e}[t]%
  \caption{\label{alg:moead}
    The \emph{decomposition-based multi-objective evolutionary algorithm} (\moead)~\protect\cite{ZhangL07} with the Chebyshev approach, maximizing a bi-objective function of size $n \in \N_{\geq 1}$.
  }
  \KwIn{A bi-objective function $f:\{0, 1\}^n\to \R^2$, a termination criterion, a mutation operator $\mathrm{mut}\colon \{0, 1\}^n \to \{0, 1\}^n$, a decomposition number $N \in N_{\geq 2}$, a weight vector $w \in [0, 1]^{N}$, as well as a neighborhood size $T \in [N]$.
  }
  \KwOut{
    A set~$P$ of incomparable best-so-far objective values.
  }
  \emph{Initialization}: set $P = \emptyset$;
  for each $i \in [N]$, choose~$x_{w_i}$ uniformly at random from $\{0, 1\}^n$;
  set $z^*_1 = \max_{i \in [N]} f_1(x_{w_i})$, $z^*_2 = \max_{i \in [N]} f_2(x_{w_i})$, and iteratively add $f(x_{w_i})$ to~$P$ if there is no $j \in [i - 1]$ such that~$x_i$ is weakly dominated by~$x_j$;
  for each $i \in [N]$, let~$B_i$ denote the indices of the~$T$ weights closest to~$w_i$\;

  \While{stopping criterion is not met}{
    \For{each subproblem $i \in [N]$}{
      \emph{Generate offspring}: $y \gets \mathrm{mut}(x_{w_i})$\;\nllabel{line:classicMoeadMutation}
      \emph{Update~$z^*$}: set $z^*_1 \gets \max(z^*_1, f_1(y))$, $z^*_2 \gets \max(z^*_2, f_2(y))$\;
      \emph{Update neighbors}: for all $j \in B_i$, if $g_{w_j}(y, z^*) \leq g_{w_j}(x_{w_j}, z^*)$, then $x_{w_j} \gets y$\;
      \emph{Update $P$}: remove all elements weakly dominated by~$f(y)$ from~$P$ and add~$f(y)$ to~$P$ if it is not weakly dominated by an element of~$P$\;
    }
  }
  return~$P$\;
\end{algorithm2e}

The \emph{decomposition-based multi-objective evolutionary algorithm} (\moead) by Zhang and Li~\cite{ZhangL07} decomposes a given bi-objective function into several single-objective subproblems, each of which is optimized in a co-evolutionary manner, that is, they are optimized in parallel and exchange solutions.
The decomposition is determined by 1) the subproblem type (the \emph{approach}) 2) the number of subproblems $N \in \N_{\geq 2}$ (the \emph{decomposition number}), and 3) the \emph{weight vector} $w \in [0, 1]^N$ that weights the objectives per subproblem.
This results in a flexible algorithm but also in a great responsibility of the user, as different parameter choices can result in drastic performance differences.

First, we explain the classic \moead (\Cref{alg:moead}) for the very commonly used~\cite{XuXM20} Chebyshev approach.
Afterward, we present our self-adjusting variant of the \moead (\Cref{alg:self-adjusting-moead}), which removes the choices of~$N$ and~$w$ and adjusts them on the fly.

\subsection{Classic MOEA/D With the Chebyshev Approach}
\label{sec:algorithms:classic-moead}

Given a decomposition number $N \in \N_{\geq 2}$ and a weight vector $w \in [0, 1]^N$, the \moead optimizes~$N$ single-objective subproblems $\{g_{w_i}\}_{i \in [N]}$, each of which is based on the original bi-objective function $f\colon \{0, 1\}^n \to \R$, the weight vector~$w$, and a \emph{reference point} $z^* \in \R^2$, which reflects the currently best known values of both objectives.
A very common way of modeling the subproblems, known as the \emph{Chebyshev approach}, is to define for all $i \in [N]$ that
\begin{align}
  \label{eq:chebyshev-approach}
  g_{w_i}\colon \{0, 1\}^n \times \R^2 \to \R,\qquad
  (x, z^*) \mapsto \max\{w_i \cdot |z^*_1 - f_1(x)|, (1 - w_i) \cdot |z^*_2 - f_2(x)|\}.
\end{align}
That is, subproblem $i \in [N]$ measures the $w_i$-weighted maximum difference between an individual and the reference point.
Each subproblem is subject to \emph{minimization}.

We call the reference point~$z^*$ \emph{optimal} if and only if its components are the maximum values of each of the objectives in the Pareto front.
For the optimal reference point and for each subproblem~$g$, we call an individual \emph{$g$-optimal} if and only if its $g$-function value is an optimum of~$g$.

The \moead aims to find the Pareto optima (of~$f$) by searching for the $g$-optima and maintaining an \emph{archive} of all currently seen objective values that are not weakly dominated.
To this end, the algorithm optimizes each subproblem following an evolutionary search.
That is, it maintains for each $i \in N$ a \emph{best-so-far} individual~$x_{w_i}$ that has the best $g_{w_i}$-value among all so-far seen individuals.

Iteratively, the \moead creates new individuals, based on~$x_{w_i}$.
These new individuals are called \emph{offspring}, and they are generated via mutation.
The input individuals for either operator are called \emph{parent(s)}.
In this paper, we consider two common mutation operators.
\emph{Standard bit mutation} creates an offspring $y \in \{0, 1\}^n$ from $x \in \{0, 1\}^n$ by choosing for each bit independently to flip it with probability~$\frac{1}{n}$, that is, for all $k \in [n]$, we have $y_k = 1 - x_k$ with probability~$\frac{1}{n}$ as well as $y_k = x_k$ otherwise.
\emph{1-bit mutation} creates an offspring $y \in \{0, 1\}^n$ from $x \in \{0, 1\}^n$ by choosing exactly one position $k \in [n]$ uniformly at random, flipping the bit at position~$k$, and copying all other bits.
That is, for all $k \in [n]$, we have $y_k = 1 - x_k$ as well as $y_k = x_k$ otherwise.

We note that the \moead classically also applies an operation known as \emph{crossover} in \mylineref{line:classicMoeadMutation} before performing mutation on the crossover result.
Crossover takes two individuals and combines them into a new individual (in a way that needs to be specified by the user, similar to mutation).
The parents for subproblem $i \in [N]$ are selected uniformly at random from the indices in~$B_i$.
Since our theoretical analysis focuses on mutation, we omit crossover in this paper, noting that it can be added to our self-adjusting variant analogously to how it is added to the classic \moead.

An offspring~$y$ is compared to all best-so-far individuals within a certain neighborhood, and~$y$ replaces each such individual~$x$ with an equal or worse fitness.
During this step, the reference point~$z^*$ is also updated, then the archive.
We note that while the neighborhood size has a large range of possible values, it was recently shown that already a minimal size of~$1$, that is, a communication just among direct neighbors, can lead to best-possible asymptotic runtimes~\cite{DoerrEKMK26}.
This suggests that the choice of this parameter may be less important as long as subproblems can exchange solutions at all.

\begin{algorithm2e}[t!]%
  \caption{\label{alg:self-adjusting-moead}
    The self-adjusting decomposition-based multi-objective evolutionary algorithm (\selfAdjustingMoead), maximizing a pseudo-Boolean bi-objective function of size $n \in \N_{\geq 1}$.
  }
  \KwIn{A bi-objective function $f\colon\{0, 1\}^n\to \R^2$, a mutation operator $\mathrm{mut}\colon \{0, 1\}^n \to \{0, 1\}^n$, and a stopping criterion.
  }
  \KwOut{
    A set~$P$ of incomparable best-so-far objective values.
  }
  \emph{Initialization}: set $W = \{0, 1\}$ and $P = \emptyset$; choose~$x_0$ and~$x_1$ uniformly at random from $\{0, 1\}^n$;
  set $z_1^* = \max(f_1(x_0), f_1(x_1)), z_2^* = \max(f_2(x_0), f_2(x_1))$;
  add~$f(x_0)$ and~$f(x_1)$ to~$P$ and remove one of them from~$P$ if one is dominated by the other;
  for each $w \in \{0, 1\}$, if $g_w(x_{1 - w}, z^*) < g_w(x_w, z^*)$, replace~$x_w$ by~$x_{1 - w}$.

  \While{stopping criterion is not met}{
    \lWhile{there are consecutive weights $w_1, w_2, w_3\in W$ with $f(x_{w_1}) = f(x_{w_2}) =\allowbreak f(x_{w_3})$} {%
      remove $w_2$ from $W$%
    }\nllabel{line:removal}
    \For{all consecutive weights $w_1, w_2\in W$ with $f(x_{w_1}) \neq f(x_{w_2})$} {
      find $w$ such that $g_w(x_{w_1}, z^*) = g_w(x_{w_2}, z^*)$ (see~\cref{eq:update})\;\nllabel{line:update}
      add $w$ to $W$ and set $x_w = \mathrm{arg\,min}_{x\in P} g_w(x, z^*)$ once we are done with this loop\;
    }
    \For{each subproblem $w\in W$}{
      \emph{Generate offspring}: $y \gets \mathrm{mut}(x_w)$\;\nllabel{line:mutation}
      \emph{Update~$z^*$}: set $z^*_1 \gets \max(z^*_1, f_1(y))$, $z^*_2 \gets \max(z^*_2, f_2(y))$\;
      \emph{Update~$x_v$}: for all $v\in W$, if $g_v(y, z^*) \leq g_v(x_v, z^*)$, then $x_v \gets y$\;\nllabel{line:optimalIndividuals}
      \emph{Update $P$}: remove all elements weakly dominated by~$f(y)$ from~$P$ and add~$f(y)$ to~$P$ if it is not weakly dominated by an element of $P$\;
    }
  }
  return~$P$\;
\end{algorithm2e}

\subsection{Self-Adjusting MOEA/D}
\label{sec:algorithms:self-adjusting-moead}
Due to the difficult and important decision of choosing a good decomposition when applying the \moead, we propose a variant that adjusts the number of subproblems as well as the respective weights automatically during the optimization process.
We call this variant the \emph{self-adjusting \moead} (\selfAdjustingMoead; \Cref{alg:self-adjusting-moead}).
We note that while we present the \selfAdjustingMoead with the Chebyshev approach (\cref{eq:chebyshev-approach}), it is applicable to any approach, although this may affect the complexity of some of the involved operations (such as the computation performed in \mylineref{line:update}).

The \selfAdjustingMoead operates mainly in the same way as the classic \moead.
That is, it creates offspring, selects best-so-far individuals, and updates the archive in an almost identical way.
The main difference is how the \selfAdjustingMoead chooses to adjust its number of subproblems and their weights.
For simplicity, we assume the full neighborhood, that is, an improvement is communicated to all subproblems.

The \selfAdjustingMoead maintains a set~$W$ of weights, each in~$[0, 1]$, with the weights~$0$ and~$1$ being initially in~$W$ and never leaving it.
In the following, we assume that~$W$ is ordered increasingly and that its elements are indexed from~$1$ to~$|W|$.\footnote{This indexation is \emph{not} reflected in \Cref{alg:self-adjusting-moead}, as it is merely useful for explaining the algorithmic idea.}
Whenever the algorithm finds a pair of \emph{consecutive} weights, that is, for $i \in [|W| - 1]$, a pair $(w_i, w_{i + 1})$, such that the objective values of the respective best-so-far individuals disagree, that is, such that $f(x_{w_i}) \neq f(x_{w_{i + 1}})$, then the \selfAdjustingMoead adds another weight, aiming to find further Pareto optima of~$f$.
Moreover, if the objective values of the best-so-far individuals of more than two consecutive weights coincide, the algorithm removes all but the two extreme involved weights, in order to reduce the cardinality of~$W$.

If there is an $i \in [|W| - 1]$ with $f(x_{w_i}) \neq f(x_{w_{i + 1}})$, the \selfAdjustingMoead adds the weight $w \in [0, 1]$ such that $g_w(x_{w_i}, z^*) = g_w(x_{w_{i + 1}}, z^*)$, since this is an extremal point where the two $g$-functions coincide for~$x_{w_i}$ and~$x_{w_{i + 1}}$, making it a promising choice, given that the objective value of~$x_{w_i}$ and of~$x_{w_{i + 1}}$ differs.
For the Chebyshev approach, we compute~$w$ directly as follows (and derive it below) as
\begin{align}
  \label{eq:update}
  w = \frac{|z^*_2 - f_2(x_{w_{i + 1}})|}{|z^*_1 - f_1(x_{w_i})| + |z^*_2 - f_2(x_{w_{i + 1}})|} .
\end{align}
To see why \cref{eq:update} holds, we first introduce for all $z \in \R$ and all $i \in \{1, 2\}$ the notation $d_i\colon \R^2 \to \R, x \mapsto |z_i - f_i(x)|$.
Thus, it holds for all $w \in [0, 1]$ and all $x \in \R^2$ that $g_w(x, z) = \max \{w \cdot d_1(x), (1 - w) \cdot d_2(x)\}$.
Note that, for all $w \in [0, 1]$, we have that the linear functions of the form $w \mapsto w \cdot d_1(x)$ all meet in $(0, 0)$, and the linear functions of the form $w \mapsto (1 - w) \cdot d_2(x)$ all meet in $(1, 0)$.
Last, for all $x \in \R^2$, let $w^*_x = d_2(x) / \bigl(d_1(x) + d_2(x)\bigr)$ denote the minimum of the function $w \mapsto g_w(x, z)$.
We argue in the following that $g_w(x_{w_i}, z^*)$ and $g_w(x_{w_{i + 1}}, z^*)$ meet at a point $w \in [0, 1]$ such that $w \geq w^*_{x_{w_i}}$ and $w \leq w^*_{x_{w_{i + 1}}}$, noting that \cref{eq:update} is the intersection of $w \mapsto w \cdot d_1(x_{w_i})$ and $w \mapsto (1 - w) \cdot d_2(x_{w_{i + 1}})$, at $w = d_2(x_{w_{i + 1}}) / \bigl(d_1(x_{w_i}) + d_2(x_{w_{i + 1}})\bigr)$.

An important observation is that, by \mylineref{line:optimalIndividuals}, we have $g_{w_i}(x_{w_i}, z^*) \leq g_{w_i}(x_{w_{i + 1}}, z^*)$ as well as $g_{w_{i + 1}}(x_{w_i}, z^*) \geq g_{w_{i + 1}}(x_{w_{i + 1}}, z^*)$.
Thus, it cannot be that $w_i < w^*_{x_{w_i}}$ and $w_{i + 1} > w^*_{x_{w_{i + 1}}}$ at the same time, as~$x_{w_i}$ is otherwise replaced by~$x_{w_{i + 1}}$ or the other way around, contradicting that $x_{w_i} \neq x_{w_{i + 1}}$, due to $f(x_{w_i}) \neq f(x_{w_{i + 1}})$.
Hence, the intersection occurs if $w_i \geq w^*_{x_{w_i}}$ or $w_{i + 1} \leq w^*_{x_{w_{i + 1}}}$.
If $w_i \geq w^*_{x_{w_i}}$, then the intersection occurs for a $w \in [w_i, 1]$ with $w \leq w^*_{x_{w_{i + 1}}}$, for an analogous argument as before.
Similarly, if $w_{i + 1} \leq w^*_{x_{w_{i + 1}}}$, then the intersection occurs for a $w \in [0, w_{i + 1}]$ with $w \geq w^*_{x_{w_i}}$, concluding how to derive \cref{eq:update}.

\paragraph{Computational complexity.}
We note that the \selfAdjustingMoead remains very similar to the classic \moead in terms of computational complexity.
The classic \moead looks for~$N$ improvements in each iteration, whereas the \selfAdjustingMoead looks for~$|W|$ improvements each iteration, noting that~$|W|$ can change during the run.
Hence, the \selfAdjustingMoead only performs more computations if $|W| > N$.

\paragraph{Setting.}
We consider the \selfAdjustingMoead both with standard bit mutation as well as with 1-bit mutation as its mutation operator.
Note that the probability to flip exactly one specific bit at position $i \in [n]$ is $\Omega(\frac{1}{n})$ for either operator, which follows from the definition of 1-bit mutation and follows for standard bit mutation by noting that the probability is $\frac{1}{n} (1 - \frac{1}{n})^{n - 1} \geq \frac{1}{e n}$.

\paragraph{Runtime.}
We define the \emph{runtime} of an \moead variant~$A$ on a bi-objective function~$f$ to be the (random) number of objective-function evaluations until the archive~$P$ of~$A$ \emph{covers} the Pareto front of~$f$, that is, until $P \supseteq F(f)$, for the first time at the end of an iteration of the main loop.
We note that it is typical to only count the function evaluations, as these are usually the most costly operations of a randomized search heuristic~\cite{DoerrN20}.
We assume that each created individual is evaluated exactly once, namely upon creation.

\paragraph{Changes for continuous problems.}
For continuous problems, the equality checks are not reasonable, as it is unlikely that equality occurs with uncountably many objective values.
Hence, we propose in such settings to define a tolerance $\varepsilon \in \R_{> 0}$ upfront and to treat all numbers with an absolute distance of at most~$\varepsilon$ as equal.

\section{Mathematical Runtime Analysis on Generalized LOTZ Functions}
\label{sec:theory}

Our main result is that the \selfAdjustingMoead (\Cref{alg:self-adjusting-moead}) computes the complete Pareto front of \emph{any} monotonically distorted \lotz function in an expected number of $O(n^4)$ objective-function evaluations.

\begin{theorem}
  \label{thm:lotz-run-time}
  The \selfAdjustingMoead optimizes \monotoneLotz in $O(n^4)$ expected function evaluations.
\end{theorem}

This bound is only worse by a factor of~$n$ than the \emph{best} known bound of the classic \moead on standard \lotz for a perfect decomposition, that is, one where the solutions to all subproblems cover the entire Pareto front.
For some monotonically distorted versions of \lotz, it is known that the common decomposition of the classic \moead with evenly spread weights does not result in a perfect decomposition anymore, which is claimed as detrimental for the algorithm~\cite{LiZZZ16}.
The reason is that the uniform decomposition does not need to match the structure of the Pareto front of the distorted function very well.
For example, consider \Cref{fig:lotz-versions:scaled-lotz}, where the distance among different Pareto optima is far from uniform.
Depending on the nature of the distortion, some Pareto optima may not be closest to any subproblem.
The \selfAdjustingMoead does not have this problem, due to its dynamic decomposition, showing that it is far more robust.

Before we prove \Cref{thm:lotz-run-time}, we prove various useful statements that we utilize for the final proof.

\Cref{lem:structure_decomp} shows that there is a partition of the weights of the \selfAdjustingMoead such that each $g$-optimal solution is also a Pareto optimum of \monotoneLotz, given that the reference point is optimal, that is, $\bigl(u(n), v(n)\bigr)$.
We note that the boundaries of the partition in \Cref{lem:structure_decomp} match \cref{eq:update} (as well as the trivial weights~$0$ and~$1$ that are always part of the algorithm), which shows that the \selfAdjustingMoead is, in theory, capable of finding weights such that all Pareto optima are discerned.

\begin{restatable}{lemma}{structureDecomp}
  \label{lem:structure_decomp}
  Consider \monotoneLotz, with monotonically increasing functions~$u$ and~$v$ as well as \cref{eq:chebyshev-approach} with $z^* = \bigl(u(n), v(n)\bigr)$.
  Let $w^*_0 \coloneqq 0$, for all $i \in [n]$, let $w^*_i \coloneqq \frac{v(n) - v(n - i)}{u(n) - u(i - 1) + v(n) - v(n - i)}$, and let $w^*_{n + 1} \coloneqq 1$.
  Then $\{[w^*_0, w^*_1]\} \cup \{(w^*_{i}, w^*_{i + 1}]\}_{i \in [n]}$ is a partition of $[0, 1]$ such that for all $i \in [0 .. n]$ and all $w \in (w^*_i, w^*_{i + 1}]$ (or all $w \in [w^*_0, w^*_1]$ in the case of $i = 0$), we have that each $(i, n - i)$-individual is the sole optimum of~$g_w$.
\end{restatable}

\begin{proof}
  Let $i \in [0 .. n]$, let $w \in (w^*_i, w^*_{i + 1}]$, but we exclude the cases $w \in \{0, 1\}$ at the moment and prove them at the very end.
  Moreover, let $x \in \{0, 1\}^n$ be the (unique) $(i, n - i)$-individual, and let $y \in \{0 , 1\}^n \smallsetminus \{x\}$.
  Ultimately, we show via a case distinction over the structure of~$y$ that $g_w(x) < g_w(y)$, which proves the claim.
  Recall that~$L$ and~$T$ denote, respectively, the number of leading ones and trailing zeros of an individual.

  Before we conduct our case distinction, we define an important value that helps us reason for all $w' \in [0, 1]$ more easily about~$g_{w'}$.
  Following \cref{eq:update}, for all $z \in \{0, 1\}^n$, we define $b_z \coloneqq \frac{v(n) - v(T(z))}{u(n) - u(L(z)) + v(n) - v(T(z))}$ as the \emph{breaking point} for~$z$ at which the maximum of~$g_{w'}$ is taken with respect to the term over the first objective, that is, for all $w' \in [0, 1]$, we have
  \begin{equation*}
    g_{w'}(z) =
    \begin{cases}
      w' \bigl(u(n) - u(L(z))\bigr)       & \text{if } w' > b_z, \\
      (1 - w') \bigl(v(n) - v(T(z))\bigr) & \text{otherwise.}
    \end{cases}
  \end{equation*}

  \textbf{Case~1: $\boldsymbol{y}$ has strictly fewer leading ones and strictly fewer trailing zeros than~$\boldsymbol{x}$.}
  Let $j = L(y) < i$ and $k < T(y) \leq n - i$.
  Each of the two terms of the~$\max$ in~$g_w(y)$ is strictly larger than the corresponding term in~$g_w(x)$.
  That is, owing to the monotonicity of~$u$ and~$v$ and to the bound on~$j$ and~$k$, we have $w \bigl(u(n) - u(i)\bigr) < w \bigl(u(n) - u(j)\bigr)$ and $(1 - w) \bigl(v(n) - v(n - i)\bigr) < (1 - w) \bigl(v(n) - v(k)\bigr)$.
  Hence, regardless of~$w$, we have $g_w(y) > g_w(x)$.

  \textbf{Case~1': $\boldsymbol{y}$ has strictly more leading ones and strictly more trailing zeros than~$\boldsymbol{x}$.}
  This is the same case as Case~1 when swapping~$x$ and~$y$.

  \textbf{Case~2: $\boldsymbol{y}$ has strictly fewer leading ones than~$\boldsymbol{x}$ and at least as many trailing zeros as~$\boldsymbol{x}$.}
  Let $j = L(y) < i$ and $k = T(y) \geq n - i$.
  We show that $b_y \leq w^*_i$ via a case distinction with respect to whether the denominator of~$b_y$ is greater than that of~$w^*_i$.
  If $u(n) - u(j) + v(n) - v(k) \geq u(n) - u(i - 1) + v(n) - v(n - i)$, then by the monotonicity of~$v$ and by $k \geq n - i$, it follows that the numerator of~$b_y$ is at most that of~$w^*_i$, that is, $v(n) - v(k) \leq v(n) - v(n - i)$.

  If $u(n) - u(j) + v(n) - v(k) < u(n) - u(i - 1) + v(n) - v(n - i)$, then by the monotonicty of~$u$ and by $j \leq i - 1$, we have $b_y = 1 - \frac{u(n) - u(j)}{u(n) - u(j) + v(n) - v(k)} < 1 - \frac{u(n) - u(i - 1)}{u(n) - u(i - 1) + v(n) - v(n - i)} = w^*_i$.

  We continue with a case distinction with respect to~$w$, recalling that we only consider values $w > w^*_i$.
  Due to $w^*_i \geq b_y$, this implies that $w > b_y$ in the following and thus that $g_w(y) = w \bigl(u(n) - u(j)\bigr)$.

  If $w > b_x$, then $g_w(x) = w \bigl(u(n) - u(i)\bigr)$.
  Hence, owing to the monotonicity of~$u$ and to $i > j$, it follows that $g_w(y) > g_w(x)$.

  Last, if $w \leq b_x$, then $g_w(x) = (1 - w) \bigl(v(n) - v(n - i)\bigr)$.
  Note that $g_w(y) > g_w(x)$ holds if and only if $w > \frac{v(n) - v(n - i)}{u(n) - u(j) + v(n) - v(n - i)} \eqqcolon W$.
  Since $w^*_i \geq W$ and since $w > w^*_i$ by assumption, we obtain $w > W$.
  This concludes Case~2.

  \textbf{Case~2': $\boldsymbol{y}$ has strictly more leading ones than~$\boldsymbol{x}$ and at most as many trailing zeros as~$\boldsymbol{x}$.}
  This is the same case as Case~2 when swapping~$x$ and~$y$.

  \textbf{Case~3: $\boldsymbol{y}$ has exactly as many leading ones as~$\boldsymbol{x}$.}
  Since we assume that $x \neq y$ and since the last leading one of~$x$ must be followed by at least one zero and thereafter at least one one in order for~$y$ to be different from~$x$, it follows that $k \coloneqq T(y) < n - i - 1$.

  We show that $b_y > w^*_{i + 1}$ via a case distinction with respect to whether the denominator of~$b_y$ is greater than that of~$w^*_i$.
  If $u(n) - u(i) + v(n) - v(k) \leq u(n) - u(i) + v(n) - v(n - i - 1)$, then by the monotonicity of~$v$ and by $k \leq n - i - 2$, it follows that the numerator of~$b_y$ is greater than that of~$w^*_{i + 1}$, that is, $v(n) - v(k) > v(n) - v(n - i - 1)$.

  If $u(n) - u(i) + v(n) - v(k) > u(n) - u(i) + v(n) - v(n - i - 1)$, then $b_y = 1 - \frac{u(n) - u(i)}{u(n) - u(i) + v(n) - v(k)} > 1 - \frac{u(n) - u(i)}{u(n) - u(i) + v(n) - v(n - i - 1)} = w^*_{i + 1}$, as claimed.

  We continue with a case distinction with respect to~$w$, recalling that we only consider values $w \leq w^*_{i + 1}$.
  Due to $w^*_{i + 1} < b_y$, we have $w < b_y$ in the following and thus $g_w(y) = (1 - w) \bigl(v(n) - v(k)\bigr)$.

  If $w \leq b_x$, then $g_w(x) = (1 - w) \bigl(v(n) - v(n - i)\bigr)$.
  Hence, owing to the monotonicity of~$v$ and to $n - i > k$, it follows that $g_w(y) > g_w(x)$.

  Last, if $w > b_x$, then $g_w(x) = w \bigl(u(n) - u(i)\bigr)$.
  Note that $g_w(y) > g_w(x)$ holds if and only if $w < \frac{v(n) - v(k)}{u(n) - u(i) + v(n) - v(k)} = b_y$, which is true, as we discussed just above.
  This concludes Case~3.

  \textbf{The weights $\boldsymbol{w \in \{0, 1\}}$.}
  We start with $w = 0$ here, which implies that we consider~$x$ to be the $(0, n)$-individual.
  In this case, regardless of $z \in \{0, 1\}^n$, we have $g_w(z) = v(n) - v(T(z))$.
  Due to the monotonicity of~$v$, this value is~$0$ if and only if $z = x$.

  For $w = 1$, we get an analogous result for~$x$ being the $(n, 0)$-individual, as, for all $z \in \{0, 1\}^n$,we have that~$g_w$ simplifies to $u(n) - u(L(z))$, which is~$0$ if and only if $z = x$, concluding the proof.
\end{proof}

\Cref{lem:pop_sz} bounds the maximum number of subproblems possible, providing us with an upper bound of the cost of each iteration.
\begin{lemma}\label{lem:pop_sz}
  Consider the \selfAdjustingMoead optimizing \monotoneLotz.
  In each iteration, the number of subproblems is at most $3n+2$.
\end{lemma}

\begin{proof}
  After the deletion step (\mylineref{line:removal} of \Cref{alg:self-adjusting-moead}), each~$f(x_w)$ appears at most twice, since we delete one occurrence of~$f(x_w)$ every time we see three consecutive ones.

  Note that the set $\{f(x_w) \mid w \in W\}$ has no element that weakly dominates another because at each time, $f(x_w)$ is the best-possible value from the archive for each $w\in W$, which implies it cannot be dominated by another $f(x_{w'})$ because $f(x_{w'})$ is in the archive. Because $f_1(x)$ can take at most $n+1$ values, we have $|\{f(x_w) \mid w\in W\}| \le n+1$. Hence after the deletion phase, $|W| \le 2 (n+1)$.

  In addition, there are at most $n$ weights that are added to~$W$ because there are at most $n$ consecutive weights in $W$ that have different values of $f(x_w)$. This gives a population size of at most $3n + 2$ after adding the weights.
\end{proof}

We split the analysis of our main result into two phases.
Phase~$1$ considers the number of iterations until the reference point is optimal.
Phase~$2$ starts immediately once phase~$1$ ends and lasts until the archive covers the entire Pareto front.

\paragraph{Phase~1.}
Since solutions are exchanged among all subproblems, the $g$-optima for weights~$0$ and~$1$ are quickly found, which implies that the reference point is optimized.

\begin{restatable}{lemma}{referencePoint}
  \label{thm:reference-point}
  Consider the \selfAdjustingMoead optimizing \monotoneLotz.
  The optimal reference point is found in $O(n^2)$ expected iterations.
\end{restatable}

We prove \Cref{thm:reference-point} by showing that the individuals corresponding to weights~$0$ and~$1$ optimize functions~$g_0$ and~$g_1$ until reaching their respective optimal individuals~$1^n$ and~$0^n$.
We measure progress, respectively, in the first or in the second objective, which is monotone.
In expectation, improving such an objective takes $O(n)$ iterations, since one specific bit needs to flip.
The result follows since we require at most~$n$ improvements.
Details are deferred to the appendix.

\paragraph{Phase~2.}

In our analysis, we define fitness levels, which show how far away the algorithm is from covering the Pareto front. The higher the fitness level, the closer the algorithm is to the Pareto front. We then show that the fitness levels increase with sufficiently high probability.

\Cref{lem:progress} shows that as long as the Pareto front is not fully covered, there is always a weight $w\in W$ such that~$x_w$ is not optimal for $g_w$, leaving room for a fitness level improvement.

\begin{restatable}{lemma}{progress}
  \label{lem:progress}
  Consider the \selfAdjustingMoead optimizing \monotoneLotz, and consider an iteration in which the Pareto front is not covered.
  Then before the mutation step (\mylineref{line:mutation} in \Cref{alg:self-adjusting-moead}), there exists an individual $x_w$ that is not $g_w$-optimal.
\end{restatable}

We prove \Cref{lem:progress} via a case distinction.
If a non-optimal~$x_w$ exists at the beginning of the iteration, weight~$w$ is only deleted if other individuals have matching objectives, which means that at least one other individual is non-optimal, by \Cref{lem:structure_decomp}, and remains.
If all current individuals are optimal, since the Pareto front is not covered, there exist at least two neighboring weights with differing objective.
Hence, a new weight with a non-optimal individual is added.
Details are deferred to the appendix.

Using a non-optimal individual from \Cref{lem:progress}, the current fitness level quickly increases via mutation.
\begin{lemma}
  \label{lem:fit_incr}
  Consider the \selfAdjustingMoead optimizing \monotoneLotz, and an iteration such that the Pareto front is not covered yet.
  Then the current fitness level increases by at least~$1$ in $O(n)$ expected iterations.
  Moreover, when starting with an optimal reference point, then the \selfAdjustingMoead finds the entire Pareto front in $O(n^3)$ expected iterations.
\end{lemma}

\begin{proof}
  We define the fitness level at the beginning of an iteration of the \selfAdjustingMoead as the number $k \in \N$ of values of the objective function~$f$ that are weakly dominated by the archive.
  Note that $k$ is always positive and at most $(n+1)^2$ as the pair $(L(x), T(x))$ can take at most $(n+1)^2$ values, which is also an upper bound for $f(x)$.
  Moreover,~$k$ cannot decrease during the run of the algorithm. Last, if the fitness level is maximal, then the Pareto front has been found.

  We show that it takes $O(n)$ expected iterations to increase each fitness level, corresponding to $O(n^3)$ expected iterations to reach a maximal fitness level, at which point the whole Pareto is covered.

  If the Pareto front is not fully covered, by \Cref{lem:progress}, at every mutation phase, there is at least one $w$ such that $x_w$ is not optimal for $g_w$. This means that with probability $\Omega(1/n)$, individual~$x_w$ is mutated into $y$ such that $g_w(y) < g_w(x_w)$, as it suffices to either increase the number of leading ones or the number of trailing zeros of~$x_w$ by one (which can be done with probability $\Omega(1/n)$). Furthermore, $y$ cannot be weakly dominated by an individual in the archive because otherwise we would have $x_w$ equal to something better from the archive, which is impossible. Hence, we add $y$ to the archive, and it weakly dominates itself, which was not the case before. Now if $X$ is the random variable giving the number of iterations before the fitness level increases, $X$ follows a geometric distribution of parameter $p = \Omega(1/n)$, resulting in $E[X] = O(n)$. Consequently, the number of weakly dominated elements increases by at least one in $O(n)$ expected iterations of the algorithm as required.
\end{proof}

\paragraph{Combining both phases.}
Combining the results of both phases and accounting for the maximum cost of an iteration by \Cref{lem:pop_sz} yields immediately our main result.

\begin{proof}[Proof of \Cref{thm:lotz-run-time}]
  By \Cref{thm:reference-point}, the first phase runs in $O(n^2)$ iterations, and from the population size being $O(n)$ (\Cref{lem:pop_sz}), we have a runtime of $O(n^3)$ function evaluations by multiplying the number of iterations by the upper bound on the number of individuals mutated per iteration.

  By \Cref{lem:fit_incr}, the second phase runs in $O(n^3)$ iterations, and from the population size being $O(n)$ (\Cref{lem:pop_sz}), we have a runtime of $O(n^4)$ function evaluations.

  We obtain a total runtime of $O(n^4)$ function evaluations by summing the runtime of both phases.
\end{proof}

\section{Conclusion}
\label{sec:conclusion}

We developed a new way how the \moead can automatically adjust the number of subproblems and the weight vectors defining them. Compared to previous such algorithms, ours is simple enough to admit proven performance guarantees (which are, in fact, the first proven performance guarantees for a dynamic multi-objective evolutionary algorithm).

We proved that our \selfAdjustingMoead computes the Pareto front of any monotonically distorted \lotz functions in an expected number of $O(n^4)$ function evaluations, clearly overcoming the difficulties detected in~\cite{LiZZZ16}.

In terms of generalization, we believe that a decomposition similar to the one presented in \Cref{lem:structure_decomp} is also possible for other established theory benchmarks, such as \textsc{OneMinMax}~\cite{GielL10} and \textsc{OneJumpZeroJump}~\cite{ZhengD23ecj}.
Moreover, we ran experiments for $0$/$1$-\textsc{Knapsack} instances, which we report in \Cref{sec:knapsack-experiments}.
In these experiments, we compare the \selfAdjustingMoead to the classic \moead with tuned parameters, finding that the former is at least as good if not better than the latter while requiring almost no parameter tuning.
This suggests that the idea of the \selfAdjustingMoead translates well to settings that go beyond the theoretical guarantees present in this work.

In this first work, we restricted ourselves to bi-objective pseudo-Boolean optimization. Natural next steps, in order of increasing need for new ideas, are (1)~the use of our algorithm for discrete problems on search spaces other than bit strings, (2)~the extension of our algorithm to more than two objectives (this in particular needs the definition of a suitable way to insert new subproblems), and (3)~possibly an extension to continuous problems, which usually operate with a fixed number of subproblems.

\subsection*{Acknowledgments}

Martin S. Krejca is supported by the French National Agency for Research (ANR) via the JCJC grant titled MultiEDA (ANR-25-CE23-2418-01).
This research also benefited from the support of the FMJH Program Gaspard Monge for optimization and operations research and their interactions with data science.

\bibliographystyle{alpha}
\bibliography{ich_master,alles_ea_master,rest}

\newpage
\section*{NeurIPS Paper Checklist}

\begin{enumerate}

      \item {\bf Claims}
      \item[] Question: Do the main claims made in the abstract and introduction accurately reflect the paper's contributions and scope?
      \item[] Answer: \answerYes{}
      \item[] Justification: The contribution in the introduction, which is clearly highlighted, provides explicit references to all the results in the main body of the paper. Moreover, the entire introduction relates our results with the state of the art.

      \item {\bf Limitations}
      \item[] Question: Does the paper discuss the limitations of the work performed by the authors?
      \item[] Answer: \answerYes{}
      \item[] Justification: We discuss limitations in the introduction, at the end of the discussion of the sMOEA/D (\Cref{sec:algorithms:self-adjusting-moead}), directly after stating our main result (\Cref{thm:lotz-run-time}), and in the conclusion.

      \item {\bf Theory assumptions and proofs}
      \item[] Question: For each theoretical result, does the paper provide the full set of assumptions and a complete (and correct) proof?
      \item[] Answer: \answerYes{}
      \item[] Justification: Each mathematical claim is rigorously stated, and its proof is given in the paper.
            We note that the proofs of \Cref{thm:reference-point,lem:progress} are, due to space constraints, deferred to the appendix, but we still provide proof sketches in the main paper.

      \item {\bf Experimental result reproducibility}
      \item[] Question: Does the paper fully disclose all the information needed to reproduce the main experimental results of the paper to the extent that it affects the main claims and/or conclusions of the paper (regardless of whether the code and data are provided or not)?
      \item[] Answer: \answerNA{}
      \item[] Justification: This paper does not contain any experiments.

      \item {\bf Open access to data and code}
      \item[] Question: Does the paper provide open access to the data and code, with sufficient instructions to faithfully reproduce the main experimental results, as described in supplemental material?
      \item[] Answer: \answerNA{}
      \item[] Justification: This paper does not contain any experiments.

      \item {\bf Experimental setting/details}
      \item[] Question: Does the paper specify all the training and test details (e.g., data splits, hyperparameters, how they were chosen, type of optimizer) necessary to understand the results?
      \item[] Answer: \answerNA{}
      \item[] Justification: This paper does not contain any experiments.

      \item {\bf Experiment statistical significance}
      \item[] Question: Does the paper report error bars suitably and correctly defined or other appropriate information about the statistical significance of the experiments?
      \item[] Answer: \answerNA{}
      \item[] Justification: This paper does not contain any experiments.

      \item {\bf Experiments compute resources}
      \item[] Question: For each experiment, does the paper provide sufficient information on the computer resources (type of compute workers, memory, time of execution) needed to reproduce the experiments?
      \item[] Answer: \answerNA{}
      \item[] Justification: This paper does not contain any experiments.

      \item {\bf Code of ethics}
      \item[] Question: Does the research conducted in the paper conform, in every respect, with the NeurIPS Code of Ethics \url{https://neurips.cc/public/EthicsGuidelines}?
      \item[] Answer: \answerYes{}
      \item[] Justification: We reviewed the NeurIPS Code of Ethics carefully and concluded that none of our results or assumptions violate the code.

      \item {\bf Broader impacts}
      \item[] Question: Does the paper discuss both potential positive societal impacts and negative societal impacts of the work performed?
      \item[] Answer: \answerNA{}
      \item[] Justification: Our main contribution provides a robust alternative to existing general-purpose optimization heuristics.
            As such, the efficiency with which certain results can be obtained may change, but our contribution does not allow approaching problems that could not be approached before.

      \item {\bf Safeguards}
      \item[] Question: Does the paper describe safeguards that have been put in place for responsible release of data or models that have a high risk for misuse (e.g., pre-trained language models, image generators, or scraped datasets)?
      \item[] Answer: \answerNA{}
      \item[] Justification: This paper does not deal with data or models.

      \item {\bf Licenses for existing assets}
      \item[] Question: Are the creators or original owners of assets (e.g., code, data, models), used in the paper, properly credited and are the license and terms of use explicitly mentioned and properly respected?
      \item[] Answer: \answerNA{}
      \item[] Justification: This paper does not use existing assets.

      \item {\bf New assets}
      \item[] Question: Are new assets introduced in the paper well documented and is the documentation provided alongside the assets?
      \item[] Answer: \answerNA{}
      \item[] Justification: We do not provide new assets, but we provide pseudocode for our newly introduced algorithm variant (\Cref{alg:self-adjusting-moead}), and we discuss its design and motivation (\Cref{sec:algorithms:self-adjusting-moead}).

      \item {\bf Crowdsourcing and research with human subjects}
      \item[] Question: For crowdsourcing experiments and research with human subjects, does the paper include the full text of instructions given to participants and screenshots, if applicable, as well as details about compensation (if any)?
      \item[] Answer: \answerNA{}
      \item[] Justification: This paper neither involves crowdsourcing nor research with human subjects.

      \item {\bf Institutional review board (IRB) approvals or equivalent for research with human subjects}
      \item[] Question: Does the paper describe potential risks incurred by study participants, whether such risks were disclosed to the subjects, and whether Institutional Review Board (IRB) approvals (or an equivalent approval/review based on the requirements of your country or institution) were obtained?
      \item[] Answer: \answerNA{}
      \item[] Justification: This paper neither involves crowdsourcing nor research with human subjects.

      \item {\bf Declaration of LLM usage}
      \item[] Question: Does the paper describe the usage of LLMs if it is an important, original, or non-standard component of the core methods in this research? Note that if the LLM is used only for writing, editing, or formatting purposes and does \emph{not} impact the core methodology, scientific rigor, or originality of the research, declaration is not required.
      \item[] Answer: \answerNA{}
      \item[] Justification: The research conducted in this paper does not involve LLMs.

\end{enumerate}

\newpage
\appendix
\section{Appendix for ``A Dynamic Decomposition Strategy for the MOEA/D With Proven Performance Guarantees''}

This appendix contains the proofs of \Cref{thm:reference-point,lem:progress}.
For the sake of convenience, we repeat the lemmas before we state their proof.
Moreover, \Cref{sec:experiments} details empirical results.

\referencePoint*

\begin{proof}
  First, by \emph{optimum of a weight $w$}, we mean an individual $x$ that minimizes $g_w(x, z^*)$.

  In the first phase, the individuals corresponding to weights~$0$ and~$1$ optimize functions $g_0$ and $g_1$ until reaching their respective optimal individuals $1^n$ and $0^n$, thus finding the reference point $z^* = (n, n)$.

  Notice that weights~$0$ and~$1$ are in $W$ at the beginning and are always in $W$ throughout the execution of the algorithm. Moreover, $x_0$ optimizes $g_0\colon (x, z^*)\mapsto z^*_2 - v(T(x_0))$ while~$x_1$ optimizes $g_1\colon (x, z^*)\mapsto z^*_1 - u(L(x_1))$.

  Notice that all mutations that increase $T(x_0)$ (resp. $L(x_1)$) decrease the objective function of $x_0$ (resp. $x_1)$. Hence, the resulting offspring replaces $x_0$ (resp. $x_1$). All mutations that decrease $T(x_0)$ (resp. $L(x_1)$) have a negative impact on their respective objective functions and, thus, do not result in the offspring replacing $x_0$ (resp. $x_1$). It follows that $T(x_0)$ and $L(x_1)$ can only increase. They increase with probability $\Omega(1/n)$ because the number of leading ones or number of leading zeros is increased by at least one by flipping the first zero or the last one respectively, happening with probability $\Omega(1/n)$ with standard bit mutation. It follows from the fitness level method that in $O(n^2)$ expected iterations of the algorithm, $z^*$ is found.
\end{proof}

\progress*

\begin{proof}
  After the first phase, for all $w\in W$, the best-so-far individual $x_w$ optimizes the function $g_w\colon x\mapsto \max\bigl\{w \bigl(u(n) - u(L(x))\bigr), (1 - w) \bigl(v(n) - v(T(x))\bigr)\bigr\}$ because $z^* = (u(n), v(n))$.
  We omit $z^*$ from the parameter of the function, as it is optimal and thus does not change.

  First, suppose at the beginning of an iteration there is such a sub-optimal $x_w$. Then there is a weight $w$ such that $x_w$ is not $g_w$-optimal. Now suppose $w$ is deleted (which occurs before the mutation), then there are two weights $w_1 < w < w_2$ with $f(x_{w_1}) = f(x_{w_2}) = f(x_w)$ that are not deleted this iteration. Suppose for the sake of contradiction that $x_{w_1}$ and $x_{w_2}$ are both optimal, then from the interval structure of the $g$-optima we showed in \Cref{lem:structure_decomp}, it follows that $x_w$ is also optimal for $g_w$, which is a contradiction. Thus at least one of~$x_{w_1}$ and~$x_{w_2}$ is not optimal, showing the claim.

  Now suppose that for all $w\in W$, $x_w$ is $g_w$-optimal. Because not the whole Pareto front is covered and the population consists of weights $0$ and $1$, there exist two consecutive weights in the population with a number of leading ones that differ by at least 2. Let $w_1$ and $w_2$ be those weights.
  These weights fall each into one of the intervals of the partition of \Cref{lem:structure_decomp}.
  That is, using the notation of \Cref{lem:structure_decomp}, there exist $a, b\le n$ with $b\ge a+2$ such that $w_1 \in (w_a, w_{a + 1}]$ (or $w_1 = 0$) and $w_2\in (w_b, w_{b + 1}]$. Thus, $f(w_1) \neq f(w_2)$, which means that the \selfAdjustingMoead adds a weight~$w$ between $w_1$ and $w_2$. This implies that $w \notin (w_a, w_{a + 1}]$ (and $w \neq 0$) as~$x_{w_1}$ has a strictly smaller~$g_w$-value than~$x_{w_2}$ in this interval, due to~$x_{w_1}$ being the unique optimum in this interval, according to \Cref{lem:structure_decomp}, and as the \selfAdjustingMoead chooses~$w$ such that the values of~$g_w(x_{w_1})$ and~$g_w(x_{w_2})$ agree. Likewise, we have $w \notin (w_b, w_{b + 1}]$. Hence, $w\in (w_{a + 1}, w_b]$ which means $g_w$-optimal individuals are not in the archive as elements of the Pareto front (which is the case of $g_w$-optimal individuals). Consequently, the individual corresponding to this new weight $w$ is not $g_w$-optimal before the mutation phase.
\end{proof}

\section{Empirical Results}
\label{sec:experiments}

This appendix contains empirical results for the \selfAdjustingMoead that the reviewers asked for.
We note that these results were not peer-reviewed.
The code is available under \url{https://doi.org/10.5281/zenodo.23570076}.

\subsection{Empirical Analysis on $0$/$1$-\textsc{Knapsack}}
\label{sec:knapsack-experiments}

We complement our theoretical results from \Cref{sec:theory} by experiments that analyze the \selfAdjustingMoead in aspects not present in the theoretical analysis.
To this end, we follow the experimental setup of Zhang and Li~\cite{ZhangL07} and consider random instances of the NP-hard bi-objective $0$/$1$-\textsc{Knapsack} problem.
Zhang and Li~\cite{ZhangL07} used this setup when introducing the classic \moead, and they showed that it empirically outperforms another common multi-objective search heuristic known as GSEMO.

In total, we compare the performance of the \selfAdjustingMoead to that of the classic \moead, and we analyze the number of subproblems of the former during the run.
Our results show that the archive of the \selfAdjustingMoead is superior to that of the classic \moead.
Moreover, the number of the subproblems of the \selfAdjustingMoead grows over time, suggesting that less time is spent on improving bad solutions, and more time is spent on exploring the Pareto front.
We detail our experimental setup in \Cref{sec:experiments:setup}, and we discuss our results in \Cref{sec:experiments:results}.

We also conducted experiments for the \selfAdjustingMoead and the classic \moead on random \monotoneLotz instances, but we detail this analysis only in \Cref{sec:lotz-experiments}.
Our results support our claims from \Cref{sec:theory}.
That is, they suggest that the \selfAdjustingMoead optimizes \monotoneLotz in reasonable time by finding new elements of the Pareto front one by one.
The classic \moead gets to the Pareto front quickly but does not find all Pareto optima within reasonable time.

\subsubsection{Setup}
\label{sec:experiments:setup}


\paragraph{Objective function.}
We consider the bi-objective \zeroOneKnapsack, which is an NP-hard combinatorial maximization problem.
The objective is to select a certain number of at most $n \in \N_{\geq 1}$ items and to place them into two bags, where the profit per item may differ per bag.
Moreover, this problem is subject to capacity constraints, as each item has a weight and each bag has a certain capacity.
More formally, for all $i \in [n]$, let $x_i \in \{0, 1\}$ denote whether item~$i$ is chosen for both bags (value~$1$) or for none (value~$0$).
In addition, for all $j \in \{1, 2\}$ let $p^{(j)}_i \in \R_{\geq 0}$ denote the profit of item~$i$ in bag~$j$, and let $\omega^{(j)}_i \in \R_{\geq 0}$ denote the weight of item~$i$ in bag~$j$.
Last, let $c^{(j)} \in \R_{\geq 0}$ denote the maximum capacity of bag~$j$.
The objective is then to maximize the function $f\colon \{0, 1\}^n \to \R_{\geq 0}^2$ given for all $j \in \{1, 2\}$ by $f_j\colon x \mapsto \sum\nolimits_{i \in [n]} p^{(j)}_i \cdot x_i$, subject to $\sum_{i \in [n]} \omega^{(j)}_i \cdot x_i \leq c^{(j)}$.

During our runs, we use~$f$ as objective function, and we only generate valid instances by applying a repair mechanism when creating new solutions, as detailed in the \hyperref[par:algorithmParameters]{\emph{Algorithm parameters} paragraph}.

\paragraph{Instances.}\label{par:instances}
We consider~$10$ random, independent instances for each value of $n \in \{250, 500, 750\}$.
For each instance, we generate the profits and weights independently by drawing a value from $[10 .. 100]$ uniformly at random for each profit and for each weight.
For each bag $j \in \{1, 2\}$, we set the capacity to half the total sum of the (random) weights, that is, $c^{(j)} = \lfloor \frac{1}{2} \sum_{i \in [n]} p^{(j)}_i \rfloor$.

\paragraph{Algorithm parameters.}\label{par:algorithmParameters}
For both the classic \moead and the \selfAdjustingMoead, we use $1$-point crossover, standard bit mutation, and a neighborhood size of $T = 10$.
Each algorithm has for each problem instance a budget of $100\,000$ FEs, after which it is stopped.
For the classic \moead, we choose a uniform weight vector, that is, for all $i \in [N]$, we choose $w_i = \frac{i - 1}{N - 1}$, and, based on preliminary tests, $N = 200$.
This is close to the parameters chosen by Zhang and Li~\cite{ZhangL07} for the classic \moead.

After mutation, we apply a \textbf{repair mechanism}, guaranteeing that the offspring $y \in \{0, 1\}^n$ satisfies the capacity constraints.
Let $\widehat{w} \in [0, 1]$ be the weight of the subproblem~$g_{\widehat{w}}$ for which~$y$ was generated, and let $z^* \in \R^2$ be the reference point.
In order to satisfy for all $j \in \{0, 1\}$ the constraint $\sum_{i \in [n]} \omega^{(j)}_i \cdot y_i \leq c^{(j)}$, we iteratively remove the item from~$y$ that has the best ratio of fitness loss to weight, that is, an item $i' \in [n]$ with $y_i = 1$ that minimizes $\bigl(g_{\widehat{w}}(y^{(i' \to 0)}, z^*) - g_{\widehat{w}}(y, z^*)\bigr) / (\omega^{(1)}_{i'} + \omega^{(2)}_{i'})$, where~$y^{(i' \to 0)}$ is identical to~$y$ except that $y^{(i' \to 0)}_{i'} = 0$.
We note that the repair mechanism makes use of additional FEs when evaluating~$g_{\widehat{w}}$, which are taken into account for the termination criterion.

\paragraph{Methodology.}\label{par:methodology}
We run the classic \moead and the \selfAdjustingMoead each~$5$ times independently per instance and per value of~$n$.
During our experiments, we log 1) the archive of each algorithm after being terminated for each run, and we log 2) the number of subproblems of the \selfAdjustingMoead during each iteration.
Afterward, we compare the archives of the classic and the \selfAdjustingMoead on the same instances, using the following \emph{set coverage} metric, defined as follows.
Given two finite archives $A, B \subseteq \R^2$, the set coverage~$C_A(B)$ of~$A$ and~$B$ is the percentage of solutions of~$B$ that are weakly dominated by~$A$, that is, $C_A(B) = 100 \cdot |\{u \in B \mid \exists v \in A\colon y \succeq x\}| / |B|$.


\paragraph{Code.}\label{par:code}
All of our experiments (excluding generating the data) runs in roughly~$90$ seconds on a computer with a $10$-core CPU, using a RAM of about $200$\,MB.
The data is generated in less than a second.
We use deterministic seeds for generating random numbers.
The code itself is available under \url{https://doi.org/10.5281/zenodo.23570076}.

\subsubsection{Results}
\label{sec:experiments:results}

\begin{table}[t]
  \caption{\label{tab:empiricalResults}
    The set coverage score for comparing the archives of the classic \moead ($c$) and of the \selfAdjustingMoead ($s$), as well as the archive size of either algorithm at the end of each run.
    We report the average numbers of all~$5$ runs for all~$10$ \hyperref[par:methodology]{instances} for all~$3$ values of~$n$, rounding numbers.
    The last column is the average of the averages reported in the~$10$ instances to the left, and it is computed using the original numbers and rounding afterward.
    We see that~$s$ weakly dominates far more points of~$c$ than the other way around.
    For all~$3$ values of~$n$, these results are statistically significant, according to a one-sided Mann--Whitney~$U$ test on the~$10$ averages per instance before rounding.
    The final archive sizes are roughly the same between both algorithms.
    Please also refer to \Cref{sec:experiments:results}.
    We report on the standard deviation of these values in \Cref{tab:standard-deviations}.
  }
  \begin{tabular}{rl*{11}{r}}
    $n$   & metric   & \multicolumn{10}{c}{instance} & avg                                                                                                                                                                    \\
    \toprule
    $250$ & $C_c(s)$ & \num{10.9602}                 & \num{11.5987}  & \num{06.4359}  & \num{15.2419} & \num{17.0742}  & \num{12.1732}  & \num{02.87912} & \num{16.0182} & \num{10.6713} & \num{16.8333}  & \num{11.4503}    \\
          & $C_s(c)$ & \num{59.958}                  & \num{64.3977}  & \num{79.4427}  & \num{55.5283} & \num{59.2649}  & \num{65.2207}  & \num{82.9415}  & \num{66.0128} & \num{78.1889} & \num{65.2154}  & \num{67.8839}    \\
    \cmidrule{2-13}
          & $|P_c|$  & \num{27.8}                    & \num{29.2}     & \num{30}       & \num{27.4}    & \num{34.2}     & \num{33.2}     & \num{32}       & \num{30.2}    & \num{23.6}    & \num{30.4}     & \num{29.8000000} \\
          & $|P_s|$  & \num{27.2}                    & \num{24}       & \num{27.8}     & \num{25.2}    & \num{26.4}     & \num{32.4}     & \num{26.6}     & \num{30.6}    & \num{23.8}    & \num{28.2}     & \num{27,2200000} \\
    \midrule
    $500$ & $C_c(s)$ & \num{06.43333}                & \num{06.10526} & \num{26.0417}  & \num{0}       & \num{05.12727} & \num{06.36364} & \num{18.303}   & \num{0}       & \num{16.6986} & \num{06.}      & \num{09.}        \\
          & $C_s(c)$ & \num{75.8579}                 & \num{68.4985}  & \num{61.2468}  & \num{90.2857} & \num{86.7435}  & \num{73.6613}  & \num{68.0849}  & \num{80.7163} & \num{69.7003} & \num{86.4569}  & \num{76.}        \\
    \cmidrule{2-13}
          & $|P_c|$  & \num{24.6}                    & \num{23.6}     & \num{29.6}     & \num{26}      & \num{22.4}     & \num{25}       & \num{21}       & \num{32.4}    & \num{24.8}    & \num{27.6}     & \num{25,7000000} \\
          & $|P_s|$  & \num{21}                      & \num{23.4}     & \num{21.6}     & \num{22}      & \num{22.6}     & \num{18.8}     & \num{19.6}     & \num{20}      & \num{19.6}    & \num{21.4}     & \num{21,0000000} \\
    \midrule
    $750$ & $C_c(s)$ & \num{17.7778}                 & \num{0}        & \num{06.66667} & \num{0}       & \num{13.75}    & \num{02.30769} & \num{13.6364}  & \num{10.8571} & \num{22.2051} & \num{01.05263} & \num{08.82534}   \\
          & $C_s(c)$ & \num{68.0357}                 & \num{89.6754}  & \num{83.4921}  & \num{90.101}  & \num{86.0504}  & \num{87.5556}  & \num{70.2778}  & \num{82.3663} & \num{72.4501} & \num{88.6365}  & \num{81.86409}   \\
    \cmidrule{2-13}
          & $|P_c|$  & \num{20.4}                    & \num{20}       & \num{20}       & \num{20}      & \num{14}       & \num{15}       & \num{22.4}     & \num{16.2}    & \num{20}      & \num{18.8}     & \num{18,6800000} \\
          & $|P_s|$  & \num{16.2}                    & \num{18.2}     & \num{15.4}     & \num{17.4}    & \num{16.2}     & \num{17.2}     & \num{14.2}     & \num{17.4}    & \num{14}      & \num{16.8}     & \num{16,3000000}
  \end{tabular}
\end{table}

\begin{table}[t]
  \caption{\label{tab:standard-deviations}
    The standard deviation of the values from \Cref{tab:empiricalResults}.
    Each value below is the average of the~$5$ standard deviations of the~$5$ runs per instance.
    We only report rounded values.
    We see that the average difference of the archive sizes seen in \Cref{tab:empiricalResults} is within one (average) standard deviation, which means that they are likely due to random variances.
  }
  \begin{tabular}{rl*{11}{r}}
    $n$   & metric   & \multicolumn{10}{c}{instance} & avg                                                                                                                                                                    \\
    \toprule
    $250$ & $C_c(s)$ & \num{13.6154}                 & \num{19.2868}  & \num{09.42125} & \num{13.603}   & \num{19.9318}  & \num{13.2427}  & \num{02.85241} & \num{17.6286} & \num{11.6288} & \num{10.7548}  & \num{13.19656}  \\
          & $C_s(c)$ & \num{19.19}                   & \num{23.6707}  & \num{17.2569}  & \num{17.6453}  & \num{21.505}   & \num{17.6737}  & \num{10.9134}  & \num{28.5729} & \num{16.7913} & \num{20.5977}  & \num{19.38169}  \\
    \cmidrule{2-13}
          & $|P_c|$  & \num{3.05941}                 & \num{2.31517}  & \num{5.21536}  & \num{5.004}    & \num{5.74108}  & \num{7.19444}  & \num{2.09762}  & \num{5.3066}  & \num{2.87054} & \num{5.004}    & \num{4.3808220} \\
          & $|P_s|$  & \num{6.93974}                 & \num{3.03315}  & \num{2.85657}  & \num{7.35935}  & \num{5.53534}  & \num{2.05913}  & \num{4.45421}  & \num{3.00666} & \num{3.81576} & \num{7.16659}  & \num{4.6226500} \\
    \midrule
    $500$ & $C_c(s)$ & \num{10.9034}                 & \num{08.05521} & \num{32.0996}  & \num{0}        & \num{06.30089} & \num{12.7273}  & \num{28.4019}  & \num{0}       & \num{31.1741} & \num{12.}      & \num{14.16624}  \\
          & $C_s(c)$ & \num{16.2787}                 & \num{19.4605}  & \num{27.5696}  & \num{08.00567} & \num{12.2303}  & \num{18.5534}  & \num{21.3893}  & \num{16.26}   & \num{34.5437} & \num{17.4816}  & \num{19.17728}  \\
    \cmidrule{2-13}
          & $|P_c|$  & \num{2.6533}                  & \num{6.85857}  & \num{8.77724}  & \num{5.09902}  & \num{3.87814}  & \num{2.52982}  & \num{5.09902}  & \num{2.05913} & \num{4.16653} & \num{4.58694}  & \num{4.5707710} \\
          & $|P_s|$  & \num{4.24264}                 & \num{2.57682}  & \num{5.88558}  & \num{6.3561}   & \num{2.05913}  & \num{3.96989}  & \num{7.03136}  & \num{2.68328} & \num{2.57682} & \num{2.33238}  & \num{3.9714000} \\
    \midrule
    $750$ & $C_c(s)$ & \num{27.9771}                 & \num{0}        & \num{13.3333}  & \num{0}        & \num{21.7945}  & \num{04.61538} & \num{19.4978}  & \num{15.5865} & \num{27.8266} & \num{02.10526} & \num{13.27364}  \\
          & $C_s(c)$ & \num{32.4902}                 & \num{09.83095} & \num{30.6904}  & \num{12.2285}  & \num{17.5435}  & \num{10.1932}  & \num{31.3852}  & \num{15.9158} & \num{27.2856} & \num{10.1191}  & \num{19.76825}  \\
    \cmidrule{2-13}
          & $|P_c|$  & \num{3.38231}                 & \num{3.40588}  & \num{2.60768}  & \num{4.93964}  & \num{5.65685}  & \num{1.89737}  & \num{2.33238}  & \num{5.3066}  & \num{6.0663}  & \num{3.37046}  & \num{3.8965470} \\
          & $|P_s|$  & \num{1.93907}                 & \num{3.54401}  & \num{3.07246}  & \num{3.44093}  & \num{2.56125}  & \num{5.1923}   & \num{5.74108}  & \num{4.63033} & \num{3.57771} & \num{3.76298}  & \num{3.7462120}
  \end{tabular}
\end{table}

We report two results.
First, we compare the archive of the \selfAdjustingMoead to that of the classic \moead (\Cref{tab:empiricalResults}).
Afterward, we analyze the number of subproblems of the \selfAdjustingMoead over time (\Cref{fig:number-of-subproblems}).

\paragraph{Comparing archives.}
\Cref{tab:empiricalResults} reports on the average of the~$5$ independent runs per instance and per value of~$n$.
In terms of the set coverage measure~$C$, that is, the relative number of individuals in one archive weakly dominating that of the other, we see that the archive of the \selfAdjustingMoead always weakly dominates far more points of the archive of the classic \moead than the other way around.
This difference is statistically significant for a $p$-value of~$0.01$ for all three values of~$n$.
This suggests that the \selfAdjustingMoead generally finds better solutions in terms of weak domination.

Moreover, the table reports on the average size of the final archive per instance.
Here, we see that the classic \moead has a slightly larger archive on average than the \selfAdjustingMoead, but the sizes are often comparable between both algorithms.
The difference is most likely a result of the random creation of the offspring, as the differences are within one standard deviation of the data (reported in \Cref{tab:standard-deviations}).
In combination with the set coverage above, these results suggest that the \selfAdjustingMoead outperforms the classic \moead in parts of the search space that they both explore.

\paragraph{Number of subproblems.}
\Cref{fig:number-of-subproblems} shows how the number of subproblems of the \selfAdjustingMoead evolves over time.
Although the figure only shows two examples, we note that the figures for all instances follow the same qualitative trend.
Initially, the \selfAdjustingMoead operates with not so many subproblems.
This means that many solutions weakly dominate each other, not requiring more subproblems.
This strategy works well until it gets harder to find improving solutions, most likely because the Pareto fron (or local optima) are found.
From this point on, the number of subproblems grows until progress is hard to make and the algorithm runs out of budget of FEs.

This behavior is good as it allows the algorithm to only use up few FEs early on and to use most of its budget once it gets closer to the Pareto front.
This allows it to explore the Pareto front more thoroughly, which is potentially why we see in \Cref{tab:empiricalResults} that the archive of the \selfAdjustingMoead dominates that of the classic \moead to a large degree.
In contrast, the classic \moead uses the same larger number of subproblems, namely $N = 200$, in each iteration, providing it with fewer opportunities to explore the Pareto front once it gets there.


\begin{figure}
  \begin{subfigure}[b]{0.48\textwidth}
    \includegraphics[width = \textwidth]{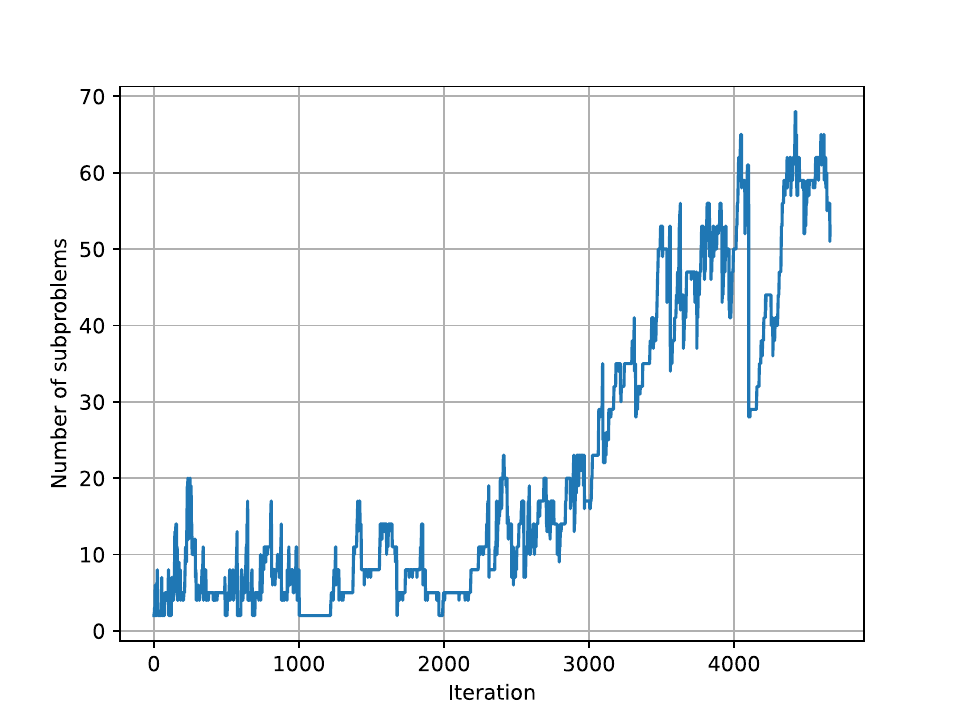}
    \caption{Instance~$0$ of size $n = 500$.}
  \end{subfigure}
  \begin{subfigure}[b]{0.48\textwidth}
    \includegraphics[width = \textwidth]{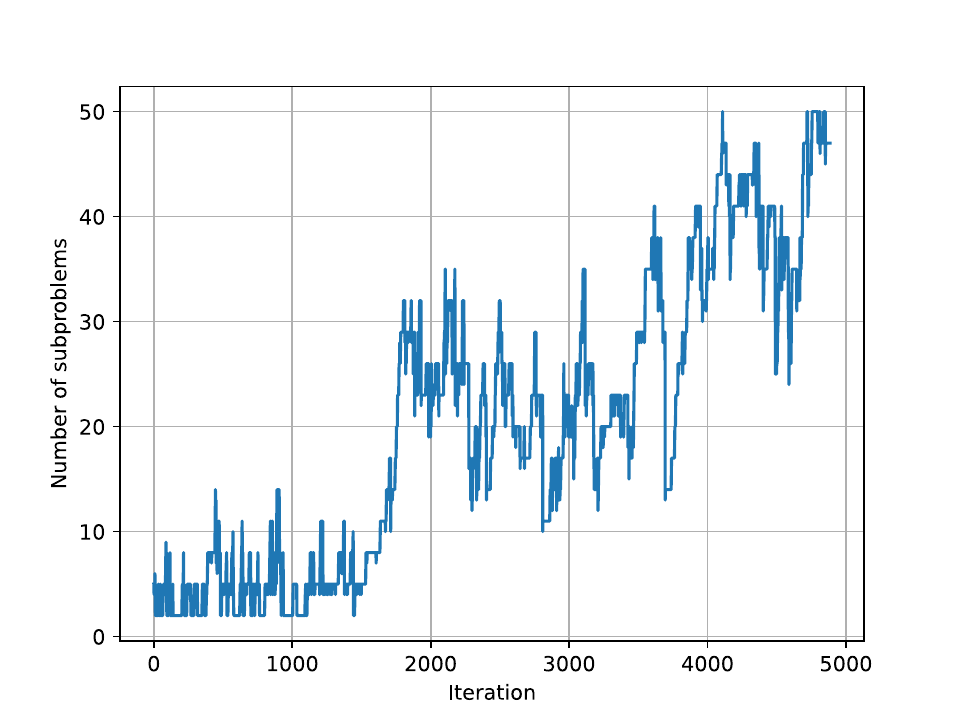}
    \caption{Instance~$0$ of size $n = 750$.}
  \end{subfigure}
  \caption{\label{fig:number-of-subproblems}
    The number of subproblems per iteration of the \selfAdjustingMoead on two random instances of the stated size.
    The number is initially small, resulting only in a few FEs, and it grows in later iterations, allowing the algorithm to explore the Pareto front more thoroughly.
    Please also refer to \Cref{sec:experiments:results}.
  }
\end{figure}

\paragraph{Conclusion.}
The \selfAdjustingMoead competes with or even outperforms a good parametrization of the classic \moead for a problem class on which the classic \moead outperformed another common multi-objective search heuristic (the GSEMO)~\cite{ZhangL07}.
The \selfAdjustingMoead does so while requiring almost no parameter choices of the user (just the neighborhood size~$T$, coupled to the choice of the crossover operator).
For parts of the Pareto front that both algorithms explore, the \selfAdjustingMoead outperforms the classic \moead.

\Cref{tab:standard-deviations} reports the standard deviation of the values from \Cref{tab:empiricalResults}.
Below, we go into detail about the additional experiments we ran.

\subsection{Performance on \monotoneLotz}
\label{sec:lotz-experiments}

We run the classic \moead and the \selfAdjustingMoead in the same setting as mentioned in \Cref{sec:algorithms:self-adjusting-moead} on random \monotoneLotz instances of size $n = 300$.
For the classic \moead, we choose $N = 4n$.
We create both $u, v\colon [0 .. n] \to [0, 1]$ by choosing $u(0) = v(0) = 0$, $u(1) = v(1) = 1$, and $n - 1$ independent random numbers from~$[0, 1]$ for either function and then sort them ascendingly.
The resulting sorted ranges of length $n + 1$ represent the ranges of~$u$ and~$v$, respectively.

\paragraph{Results.}
We run both algorithms on the same~$5$ random instances.

We observe that the \selfAdjustingMoead always covers the entire Pareto front within reasonable time, that is, within about $4$~minutes (using the same computational resources as mentioned in \hyperref[par:code]{the \emph{Code} paragraph} of \Cref{sec:experiments:setup}).
In contrast, the classic \moead does not cover the entire Pareto front even after $5$~hours of execution.

For either algorithm, we log how quickly it finds solutions on the Pareto front.
Both algorithms quickly find first solutions on the Pareto front and are then slowed down by finding the remaining solutions.
The \selfAdjustingMoead finds all solutions eventually but does so one by one.
Its number of subproblems grows steadily and proportionally to the size of the Pareto front covered.
The classic \moead finds the solutions to the subproblems quickly but does not manage to find all of the remaining optima within reasonable time.

\end{document}